\documentclass{article}
\usepackage{times}
\usepackage{color}
\usepackage{graphicx}
\usepackage{latexsym}
\usepackage[]{algorithm2e}
\usepackage{amsfonts}
\usepackage{amssymb}
\usepackage{amsmath}
\usepackage{xspace}
\newcommand{\moves}{\small \textsc{Moves}\xspace}

\newcommand{\para}[1]{\smallskip\noindent\textbf{#1}}

\newcommand{\alg}{{Algorithm}~\ref{alg}\xspace}

\newcommand{\algtwo}{{Algorithm}~2\xspace}

\newcommand{\La }{\mbox{$\mathcal{L}$} }

\newcommand{\A}{\mathcal{A}}

\newcommand{\gag }{G_{\A, \Ct}}
\newcommand{\gagh }{G_{\A}^h}

\newcommand{\ai}{\varphi}

\newcommand{\Js}{J}

\newcommand{\CH}{\mathrm{CH}}

\newcommand{\Ct}{\Gamma}

 \newcommand{\Dm}{\{0,1\}^{\A}}

\newcommand{\Dmc}{\mathcal{J}_{\A, \Ct}}

\newcommand{\Pf}{P}

 \newcommand{\ie}{\mbox{\textit{i.e.}, }}
\newcommand{\eg}{\mbox{\textit{e.g.}, }}

\newcommand \argmin[1] {\underset{{#1}}{\mbox{argmin}}}

\newtheorem{definition}{Definition}
\newenvironment{proof}{{\bf Proof}}{\hfill\qed\vspace{1ex}}
\def\qed{$\blacksquare$}
\newtheorem{theorem}{Theorem}
\newtheorem{proposition}[theorem]{Proposition}
 \newtheorem{lemma}[theorem]{Lemma}
\newtheorem{observation}[theorem]{Observation}
 \newtheorem{corollary}[theorem]{Corollary}
 \newtheorem{example}{Example}

\newcommand{\putaway}[1]{}

\newcommand{\new}[1]{{\color{black}#1}}
\renewcommand{\ie}{i.e., }
\renewcommand{\eg}{e.g., }
\newcommand{\marija}[1]{{\color{black} #1}}

\begin{document}

\title{Iterative Judgment Aggregation}

\author{Marija Slavkovik 
   \and Wojciech Jamroga}

\maketitle
\bibliographystyle{ecai}

\begin{abstract}
Judgment aggregation problems form a class of  collective decision-making problems represented in an abstract way,  subsuming some well known problems such as voting. A collective decision can be reached in many ways, but a direct one-step aggregation of individual decisions is arguably most studied. Another way to reach collective decisions is by iterative consensus building -- allowing each decision-maker to change their individual decision in response to the choices of the other agents until a consensus is reached. Iterative consensus building  has so far only been studied for voting problems. Here we propose an iterative judgment aggregation algorithm, based on movements in an undirected graph, and we study for which instances it terminates with  a consensus. We also compare the computational complexity of our itterative procedure with that of related judgment aggregation operators.
\end{abstract}

\section{Introduction}

Social choice aggregation methods, such as voting \cite{Nurmi10}, are of interest to artificial intelligence as methods for collective decision-making among humans and automated agents alike \cite{handbook}. Judgment aggregation problems \cite{List2010} are problems of aggregating individual judgments on a fixed set of logically related issues, called an agenda. Intuitively, an issue is a question that can be answered ``yes'' or ``no.'' Alternatively, an issue is a proposition that can be accepted or rejected. A judgment is a consistent collection of  ``yes'' and ``no''  answers, one for each issue. Judgment aggregation has been used to model collective decision-making in multi-agent systems \cite{COIN10,thesis,Synthese12new}. It is also interesting because it generalises voting problems, i.e., problems of choosing one option from a set of available options by  aggregating agents' individual preferences over these options.
A voting problem can be represented as a judgment aggregation problem under some mild conditions, see e.g.~\cite{DietrichList07,ADT2013}.

Aggregation methods produce a joint decision for a group of agents by aggregating the set of their individual decisions, called a profile, using an aggregation operator. Another approach to collective decision-making is deliberation, when agents negotiate which decisions to make. In multi-agent systems, deliberation procedures are constructed using an abstract argumentation framework to model relations between decisions, cf.~e.g.~\cite{Ontanon2008,Kok2011}. A third,
comparatively less explored  method to reach collective decisions is by {\em iterative consensus building}: each agent starts with an individual decision which she then iteratively changes in response to the individual decisions of the other agents until all agents end up supporting the same decision, \ie until a {\em consensus} is reached. While in standard aggregation all individual decisions are elicited once, forming a profile, and after the elicitation the agents can no longer change the profile, an iterative procedure allows agents to change their decisions many times, even back and forth.

The existence of judgment transformation functions, i.e., functions that transform one profile of individual judgments into another profile of individual judgments (possibly towards consensus) has been  considered by List~\cite{List2011}. It was shown that under a set of reasonable and minimal desirable conditions no transformation function can exist. Social choice aggregation theory is rife with impossibility results such as this, yet few specific aggregation operators (that violate some of the desirable conditions) are proposed. There are more voting  operators than judgment aggregation operators, which is unsurprising since voting is a much older discipline, but the number of judgment aggregation operators is also on the rise. Those include: quota-based rules \cite{DietrichL2007,premise10}, distance-based rules \cite{ConalPiggins12,EndrissG14,EndrissGP12,MillerOsherson08,Pigozzi2006,JELIA14}, generalisations of Condorcet-consistent voting rules \cite{TARK11,NehringPivato2011,PuppeNehring2011},  and  rules  based on the maximisation of some scoring function \cite{Dietrich:2013,TARK11,Zwicker11}. Deliberation and iterative consensus reaching procedures for voting problems %are not many but they 
have been explored, \eg in~\cite{MeirPRJ10,Lev:2012,GrandiADT,Hassanzadeh2013,Obraztsova15}.
However, to the best of our knowledge, there are no iterative procedures for aggregating judgments. With this work we aim to fill in the gap.

We consider all possible judgments for an agenda as vertices in a graph. The existence of an edge between judgments in the graph depends on the relations between the truth-value assignments on the same issue in the connected judgments. We define three intuitive agenda graphs.  We design an iterative consensus building algorithm which reaches consensus in the following way: In the first step of the algorithm, each agent chooses a vertex and lets the other agents know what she has chosen. In  each subsequent step  each agent independently from the other agents moves to an adjacent vertex if this move reduces her path distance to the other vertices in the profile. The agents are only allowed to   move along a shortest path towards some other agent. The moving continues until a consensus is reached (\ie when all agents end up on the same vertex). We then exploit properties of graphs to study for which  subgraphs corresponding to a profile of judgments the algorithm terminates with a consensus.

Judgment aggregation operators suffer from two shortcomings. First, they are often irresolute, \ie more than one collective decision is produced
Unlike in voting, tie-breaking in judgment aggregation is not straightforward and virtually unexplored. 
Secondly, deciding if a given judgment is among the possible ``winners'' of the aggregation is often intractable \cite{ADT2013,EndrissDeHaanAAMAS2015}.
The set of tractable aggregation functions is very limited, exceptions being \cite{DietrichL2007,premise10,EndrissG14}.
An iterative procedure clearly avoids ties when it reaches a consensus, and this is one advantage of our proposal.
We also show that our consensus-oriented procedure may offer some computational benefits when compared to standard judgment aggregation rules.

The  motivation  for our iterative procedure is both descriptive and prescriptive. On one hand, our algorithm is meant to approximate consensus formation that happens in human societies. On the other hand, our procedure can be useful for implementing artificial agents, as producing a consensual judgment is in some cases distinctly cheaper than computing the collective opinion in one step elicitation by a standard judgment aggregation procedure.

The paper is structured as follows. In Section~\ref{sec:pre} we introduce the judgment aggregation framework. In Section~\ref{sec:ag} we define three relevant agenda graphs and recall some useful concepts from graph theory. In Section~\ref{sec:algorithm} we present the algorithm for iterative judgment aggregation, and discuss necessary conditions for its termination with a consensus. In Section~\ref{sec:termination} we investigate sufficient termination conditions for each of the three agenda graphs. In Section~\ref{sec:properties} we briefly discuss the quality of the reached consensus with respect to some  judgment aggregation operators, and study the computational complexity of the algorithm. In Section~\ref{sec:rwd} we discuss the related work. In Section~\ref{sec:sum} we present our conclusions and discuss future work.

\section{Preliminaries}\label{sec:pre}

We first introduce the basic definitions of judgment aggregation.

\para{Judgments.}
Let $\La$ be a set of  propositional variables.
An {\em agenda} $\A = \{ \ai_1,  \ldots, \ai_m\}$ is a finite set $\A \subseteq \La$. The elements of $\A$ are called {\em issues}.   A {\em judgment} is a (truth assignment) function $\Js: \A \rightarrow \{0,1\} $ mapping each issue to either 0 (reject) or 1 (accept).  We write $\Dm$ as a shorthand  for  $\A \rightarrow \{0,1\}$, the space of all possible judgments for $m$ issues, \ie all sequences of length $m$ comprised of 0s and 1s.  We use $\Js(\ai)$ to denote the value assigned to $\ai \in \A$.  The Hamming distance between two judgments is defined as the number of issues on which the judgments differ, \ie $d_h(\Js^a, \Js^b) = \#\{ \ai \in \A \mid \Js^a(\ai) \neq \Js^b(\ai) \}$.

With each agenda we associate a {\em constraint} $\Ct \in \La_{\A}$, where $\La_{\A}$ is the set of well formed formulas constructed with variables from $\A$ and the logical connectives $\neg$ (negation), $\wedge$ (conjunction), $\vee$ (disjunction), and $\rightarrow$ (implication). The formula $\Ct$ is assumed not to be a contradiction. A judgment from  $\Dm$ is {\em rational} for $\Ct$ if and only if it is a model for $\Ct$ in the sense of classical propositional logic.
% For  $\Ct\in \La_{\A}$ we define recursively the notion of a judgment from $\Dm$ {\em being rational} for $\Ct$:
%\begin{itemize}
%\item  $\Js \models \ai$ for $\ai \in \A$ iff $\Js(\ai)=1$;
%\item $\Js \models \neg \Ct$ iff  it is not the case that $\Js \models \Ct$;
%\item  $\Js \models \Ct \wedge \Ct'$ iff    $\Js \models \Ct$ and $\Js \models \Ct'$.
%\end{itemize}
   For a given $\Ct \in \La_{\A}$, we define $\Dmc \subseteq \{ 0,1\}^{\A}$
   %=\{ \Js \in \{ 0,1\}^{\A} \mid \Js \models \Ct\}$
   to be the set of all {\em rational judgments}  for $\A$ and $\Ct$.

%
% \begin{example}\label{ex:running}As a running example\footnote{Our running example is the running example introduced in \cite{}.}, consider the agenda $\A=\{ \ai_1, \ai_2, \ai_3, \ai_4, \ai_5\}$ and constraint$\Ct = (( \ai_1 \wedge  \ai_3) \rightarrow  \ai_4) \wedge (( \ai_2 \wedge  \ai_3) \rightarrow  \ai_4) \wedge (\ai_4 \rightarrow \ai_3)$. The $\Dm$ has 32 judgments, while $\Dmc$ has 17 rational judgments that correspond to the nodes shown on Figure~\ref{fig:counterexample}.
% \end{example}
%
% \begin{figure}
%\centering
% \includegraphics[width=0.5\textwidth]{RunningExample}
% \vspace{-1.2cm}
% \caption{The graph $\gag^m$ for $\A$ and $\Ct$ given in Example~\ref{ex:running}. }\label{fig:counterexample}
%  \vspace{-0.8cm}
% \end{figure}

 %For each $\Js \in \Dm $ we can write the {\em corresponding judgment set} as $set(\Js) = \{ \ai \mid \ai \in \A, \Js(\ai) =1 \} \cup \{ \neg \ai \mid \ai \in \A, \Js(\ai) =0 \}$.  For each $S \subseteq set(\Js)$ we define $\ol{(S)} = \{ \neg \ai \mid \ai \in S\} \cup \{ \ai \mid \neg \ai \in S\}$. A  set $ S \subseteq set(\Js)$ is consistent iff  $S \cup \{\Ct\}$ is a consistent set in the classical propositional logic sense, and inconsistent otherwise.

\para{Agents and profiles.} Let $N=\{1, 2, \ldots, n\}$ be a finite set of {\em agents}.  A profile $\Pf = (\Js_1, \ldots,\Js_i, \ldots, \Js_n) \in \Dmc^n$ is a list of rational judgments, one for each agent. We denote $P[i] = \Js_i$
% when $\Js_i$ occurs as the $i$th agent's judgment in $\Pf$
and $\Pf_{-i} = (\Js_1, \ldots,\Js_{i-1}, \Js_{i+1}, \ldots, \Js_n)$. Further let $\{\Pf\}$ be the set of all distinct judgments that are in $\Pf$.
 We often abuse notation and write $\Js_i \in \Pf$ when  $\Pf[i] = \Js_i$. We reserve subscripted judgments, \eg  $\Js_i$,  to denote judgments that belong to some profile and the superscripted judgments, \eg $\Js^a$, $\Js^b$, to denote rational judgments that may not belong to some profile.
%
% \begin{example} Consider $\A = {\ai_1, \ai_2, \ai_3}$ with $\Ct = (\ai_1 \wedge \ai_2) \leftrightarrow \ai_3$. Let $\Js_1= (\ai_1 \mapsto 0, \ai_2 \mapsto 0, \ai_3 \mapsto 0)$, $\Js_2= (\ai_1 \mapsto 0, \ai_2 \mapsto 1, \ai_3 \mapsto 0)$, and $\Js_3= (\ai_1 \mapsto 1, \ai_2 \mapsto 0, \ai_3 \mapsto 0)$.
% % We have that $m(\Pf) =  (\ai_1 \mapsto 1, \ai_2 \mapsto 1, \ai_3 \mapsto 0)$ is not a rational judgment.
% %The profile $\Pf=(\Js_1, \Js_2, \Js_3)$ is compact because the only other judgment in $\Dmc$, $\Js_= (\ai_1 \mapsto 1, \ai_2 \mapsto 1, \ai_3 \mapsto 1 )$ is not in $Cps(\Pf)$.
%  \end{example}
A profile is {\em unanimous} if $\{ \Pf\} = \{ \Js\}$,  for some $\Js \in \Dmc$. A judgment $\Js$ is a {\em plurality judgment} in $\Pf$ if and only if $\#\{i \mid \Pf[i] = \Js\} \geq \#\{i \mid \Pf[i] \neq \Js\}$.

\begin{figure}[h!]\centering
\new{\begin{tabular}{ | l || ccc|} \hline
    & $\varphi_1$ &  $\varphi_2$ & $\varphi_3$   \\ \hline
  agent 1  & 0 & 1 & 0 \\
  agent 2  & 1 & 0 & 0 \\
  agent 3  & 1 & 1 & 1 \\ \hline \hline
  majority & 1 & 1 & 0 \\ \hline
\end{tabular}
\caption{Doctrinal paradox}
\label{fig:paradox}}
\end{figure}

\begin{example}\label{ex:doctrinal}
The quintessential example in judgment aggregation is the ``doctrinal paradox" which is described with $\A = \{ \ai_1, \ai_2, \ai_3\}$ and $\Ct= (\ai_1 \wedge \ai_2) \leftrightarrow \ai_3$. The $\Dmc = \{ (0,0,0), (0,1,0), (1,0,0), (1,1,1)\}$.
The doctrinal paradox profile is $\Pf=( (0,1,0), (1,0,0), (1,1,1))$\new{, see also Figure~\ref{fig:paradox}. Note that all the three \marija{profile} judgments are rational, but the collective judgment obtained by taking} the value for each issue assigned by a strict majority of agents, the so called majority rule, is not rational.
\end{example}

\section{Agenda Graphs}\label{sec:ag}

An {\em agenda graph} is a graph $\gag = \langle V, E\rangle$ whose nodes are judgments for some agenda $\A$, namely $V \subseteq \{ 0,1\}^{\A}$.
Given an agenda $\A$ and constraints $\Ct$ we define three agenda graphs. The {\em Hamming graph} $\gagh$ is the graph over all possible (not necessarily rational!) judgments, that connects vertices which differ on exactly one issue. The  {\em Model graph} $\gag^m$ is the graph over all {\em rational} judgments, where two judgments are adjacent iff no ``compromise" exists between them. The {\em Complete graph} $\gag^c$ is the fully connected graph over all  rational judgments. Formally:
%the {\em Hamming graph} $\gagh$, the {\em Model graph} $\gag^m$, and the {\em Complete graph} $\gag^c$, in the following way:
%
\begin{description}
\item The {\bf Hamming graph} is  $\gagh = \langle \Dm, E^h\rangle$ where  $(\Js^a,\Js^b)\in E^h$ iff $d_h(\Js^a, \Js^b)=1$.
\item The {\bf Model graph}  is   $\gag^m =\langle \Dmc, E^m\rangle$  where   $(\Js^a,\Js^b)\in E^m$ iff  there exists no judgment $\Js^c\in \Dmc$ {\em between} $\Js^a$ and $\Js^b$. A judgment $\Js^c \in \Dmc$ is between judgments $\Js^a \in \Dmc$ and  $\Js^b \in \Dmc$ when $\Js^c \neq \Js^a$, $\Js^c \neq \Js^b$, $\Js^a \neq \Js^b$  and  for every $\ai \in \A$    if $\Js^a(\ai) = \Js^b(\ai)$, then $\Js^c(\ai) = \Js^a(\ai)=\Js^b(\ai).$
\item The {\bf Complete graph} is  $\gag^c = \langle \Dmc, E^c\rangle$ where  $E^c= \Dmc \times \Dmc$.
\end{description}
%
%\new{That is, $\gagh$ is the graph over all possible (not necessarily rational!) judgments, that connects vertices which differ on exactly one issue.
%$\gag^m$ is the graph over all rational judgments, where two vertices are adjacent iff they cannot be ``reconciled'' without altering on one of the issues on which they agree.
%$\gag^c$ is the fully connected graph over all the rational judgments.
%
The agenda graphs for the doctrinal paradox of Example~\ref{ex:doctrinal} are shown in Figure~\ref{fig:doctrinalexample}.

\begin{figure}[h!]\centering
 \includegraphics[width=0.5\textwidth]{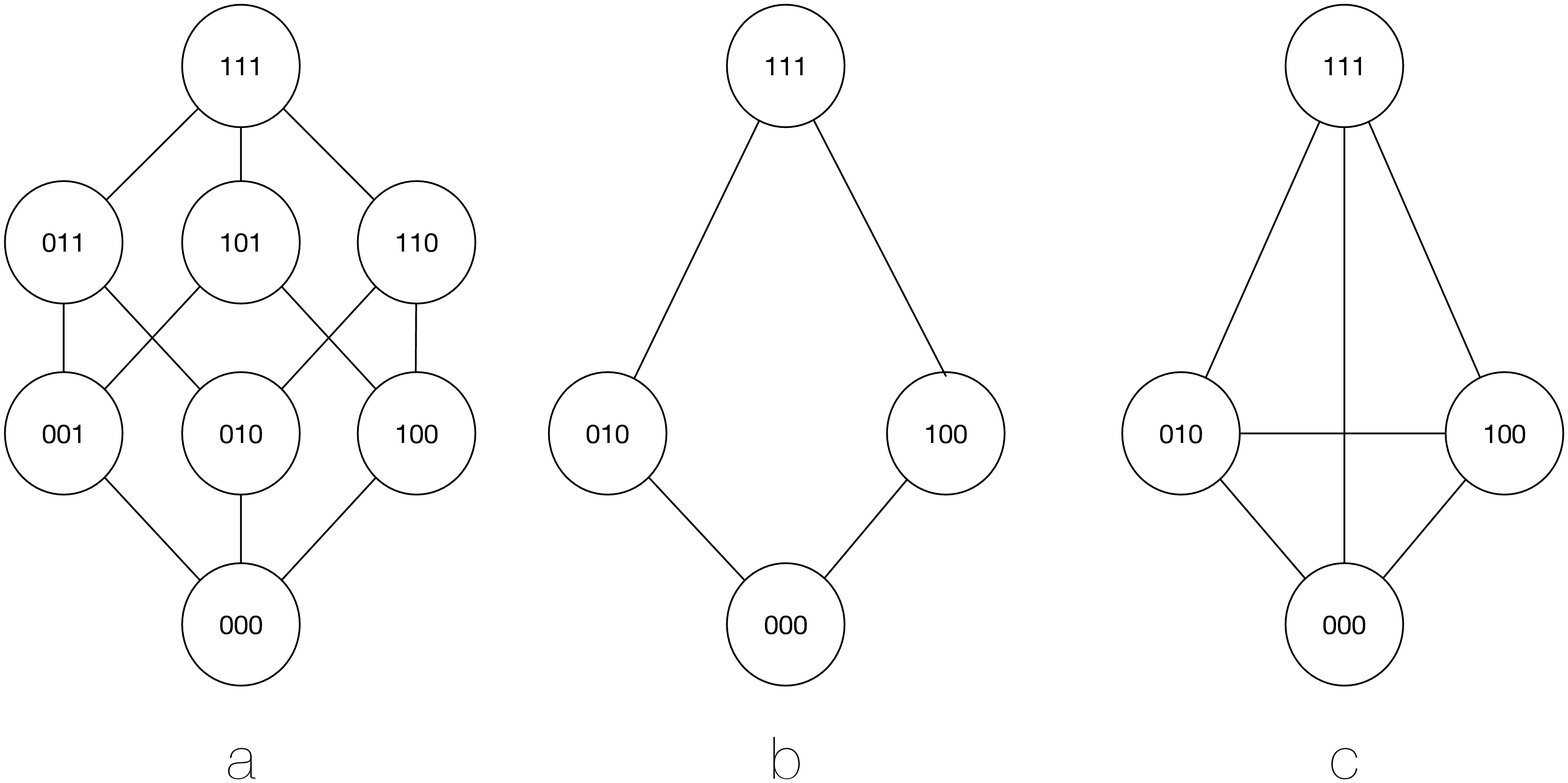}
 \caption{Agenda graphs for the doctrinal paradox:\quad
   (a) the Hamming graph $\gagh$,\quad (b) the Model graph $\gag^m$,\quad (c) the Complete graph $\gag^c$ for  $\A$ and $\Ct$ in Example~\ref{ex:doctrinal}. }\label{fig:doctrinalexample}
\end{figure}

We use $d_x$ to denote the (shortest) path distance on an agenda graph $\gag^x$. The path distance on $\gag^c$ is also known as the {\em drastic distance}: the distance between two judgments is 0 iff they are the same on all issues and 1 iff they differ on at least one issue. The path distance on $\gagh$ is the Hamming distance, and the path distance on $\gag^m$ is the Model distance introduced and formally characterized in~\cite{ConalPiggins12}. Recall that a  path distance on a graph $\gag$, as on any graph, is a distance function in the true sense since for  every $\Js^a, \Js^b, \Js^c\in V$ it satisfies:  $d(\Js^a, \Js^b)=0$ iff $\Js^a =\Js^b$,
$d(\Js^a, \Js^b)= d(\Js^b, \Js^a)$,  and $d(\Js^a, \Js^c) \leq d(\Js^a, \Js^b)  + d(\Js^b, \Js^c) $ (triangle inequality).

A graph $G' = \langle V',E' \rangle$ is a subgraph of graph $G = \langle V, E\rangle$, denoted $G' \subseteq G$,  if $V' \subseteq V$ and $E'\subseteq E$.
The $V'$-induced subgraph of a graph $G$ is the graph $G'\subseteq G$ with vertices $V'$ and edges $E'$ which satisfies that, for every pair of vertices in $V'$, they are adjacent in $G$ if and only if they are adjacent in $G'$.

We consider profile-induced subgraphs of $\gag$  and make use of their ``geometry". Therefore we define some useful concepts  for a given agenda graph $\gag=\langle V,E\rangle$ and agents $N=\{1, 2, \ldots, n\}$, following the terminology from graph theory \cite{Pelayo}.
\begin{description}
\item The {\em interval} between  a pair of vertices $\Js^a, \Js^b \in V$, denoted $I_{\A,\Ct}[\Js^a, \Js^b]$, is the set of all the judgments on all the shortest paths in $\gag$ from $\Js^a$ to $\Js^b$.
\item A subset $S \subseteq V$ is \emph{convex} if it is closed under $I_{\A,\Ct}$, namely when it includes all shortest paths between two vertices in $S$.
\item The {\em convex hull of $\Pf$} on $\gag$,  denoted $\CH(\Pf)$, is the smallest convex subset of $V$ from $\gag$ that contains $\{\Pf\}$.
\item The {\em eccentricity} of a judgment $\Js^a \in S \subseteq V$ is $e_S(\Js^a) = \max\{ d(\Js^a, \Js^b) \mid \Js^b \in S\}$, \ie the farthest that $\Js^a$ gets from any other judgment in $S$.
\item A {\em diameter} of a set $S \subseteq V$  is $mx_d(S) = \max\{ e_{S}(\Js) \mid \Js \in S\}$, namely the maximal eccentricity of a vertex in $S$. All judgments for which $mx_d(S) = e_S(\Js) $ are called {\em peripheral judgments} for $S$. For  $S = \{\Pf\}$ we call these judgments {\em peripheral judgments of a profile $\Pf$}. If for two judgments $\Js^a, \Js^b$ it holds that $d(\Js^a,\Js^b) = mx_d(\Pf)$, then these are called {\em antipodal judgments of a profile $\Pf$}.
\end{description}

\begin{example}\label{ex:hull}
Consider the Hamming graph $\gagh$ for the doctrinal paradox, presented in Figure~\ref{fig:doctrinalexample}a, and take the profile $\Pf$ as defined in Example~\ref{ex:doctrinal}. The $\CH(\Pf)$-induced subgraph of $\gagh$ is given in Figure~\ref{fig:hull}; this graph contains all the shortest paths, and nodes from $\Dmc$,  between $(0,1,0)$ and $(1,0,0)$,  between $(0,1,0)$ and $(1,1,1)$, and between $(1,0,0)$ and $(1,1,1)$. Node $(0,0,1)$ is not in this $\CH(\Pf)$-induced subgraph of $\gagh$ because this node is not on any of  the shortest paths between the profile judgments.
\begin{figure}[h!]\centering
 \includegraphics[width=0.2\textwidth]{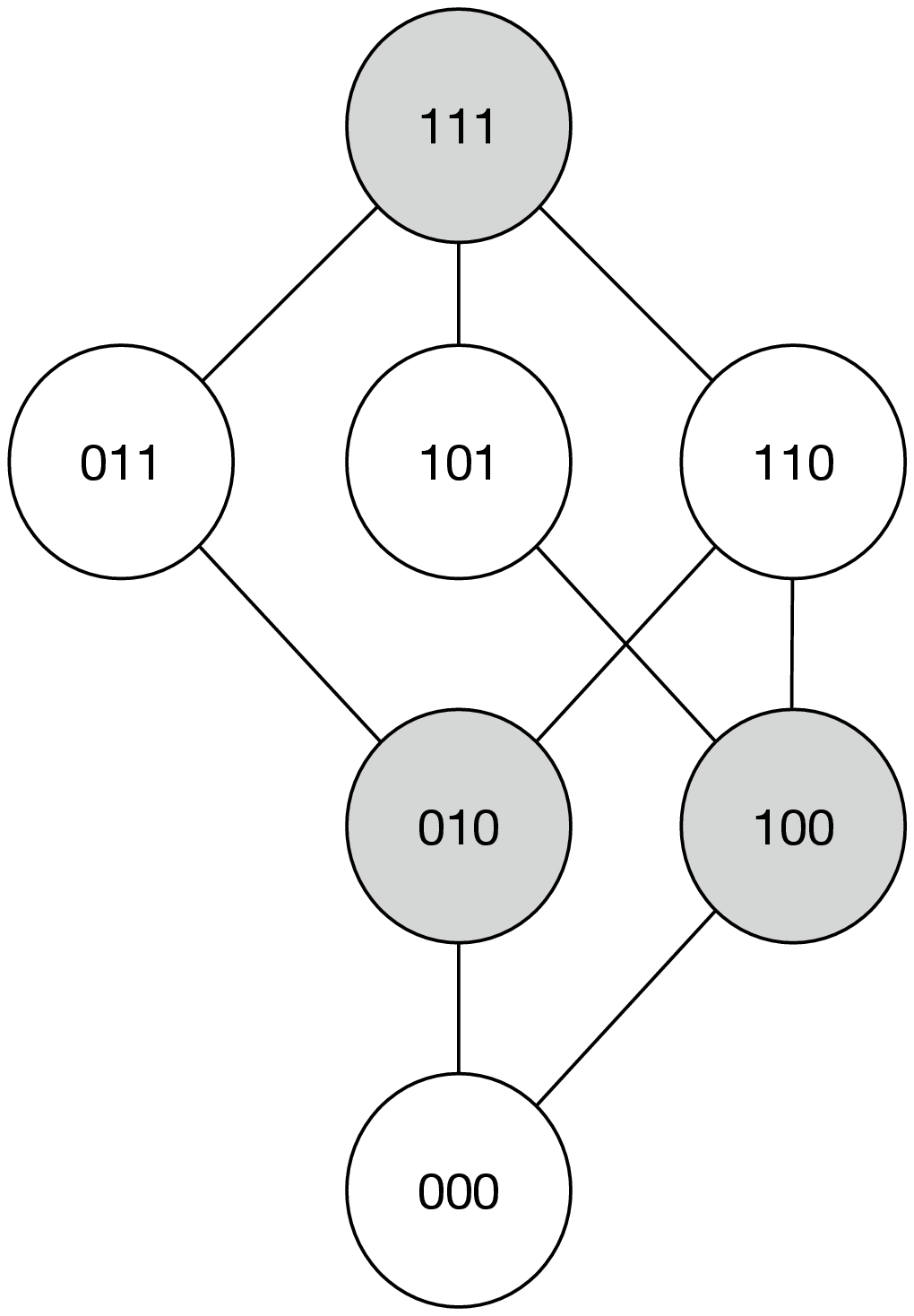}
 \caption{The $\CH(\Pf)$-induced subgraph of $\gagh$ for $\Pf$ in Example~\ref{ex:doctrinal}.}\label{fig:hull}
\end{figure}
\end{example}

We have defined the convex hull of $\Pf$ because we will build our iteration algorithm on the principle of only allowing the agents to move from their current judgment to an adjacent judgment in the hull of $\Pf$. By doing this we ensure that the agents do not disperse away from each other. The Proposition~\ref{prop:mxd} shows that if agents in profile $\Pf$ each move to a judgment in $\CH(\Pf)$, thus creating profile $\Pf'$, the diameter of the $\Pf'$ profile cannot be bigger than that of the $\Pf$.

\begin{proposition}\label{prop:mxd}
For an agenda graph $\gag$ and profile $\Pf \in \Dmc^n$, if $S \subseteq \CH(\Pf)$, then $\CH(S) \subseteq \CH(\Pf)$ and  $mx_d(\CH(S)) \leq mx_d(\CH(\Pf))$.
\end{proposition}
\begin{proof}
This property follows from the fact that $\CH$ is a finitary closure operator \cite[p.~6, Theorem~1.3]{Pelayo}. Thus, for the convex hull $\CH(\Pf)$, it holds that if $S \subseteq \CH(\Pf)$, then $\CH(S) \subseteq \CH(\CH(\Pf))$, and $\CH(\CH(\Pf)) = \CH(\Pf)$. \end{proof}

\begin{definition}
We say that the profile $\Pf$ has a $k$-cycle in $\gag$ if and only the $\CH(\Pf)$-induced subgraph of $\gag$ has a simple cycle of length $k$.
We say that $\Pf$ is a $k$-cycle in $\gag$ if and only if the $\CH(\Pf)$-induced subgraph of $\gag$ is a simple cycle in $\gag$ of length $k$.
\end{definition}

The doctrinal paradox profile from Example~\ref{ex:doctrinal} is a $4$-cycle in $\gag^m$, and it has a $6$-cycle in $\gagh$, as can be inferred from Figure~\ref{fig:doctrinalexample}.

We make the following observation.

\begin{observation}\label{col:cycle}
No profile $\Pf\in\Dmc^n$ has a 3-cycle in $\gagh$ or in $\gag^m$, regardless of $\A$ and $\Ct$.
\end{observation}

This exclusion of 3-cycles is due to the {\em betweenness} property of path distances: if $\Js^b$ is between $\Js^a$ and $\Js^c$ in the graph, then $d(\Js^a, \Js^c) = d(\Js^a, \Js^b) + d(\Js^b, \Js^c)$.

\begin{definition}
An interval $I_{\A,\Ct}[\Js^a,\Js^b]$ is {\em $\epsilon$-connected} in $\gag$ if and only if the maximal path distance between two adjacent {\bf rational} judgments in the $I_{\A,\Ct}[\Js^a,\Js^b]$-induced subgraph of $\gag$ is at most $\epsilon \in \mathbb{N}$.
A profile {\em $\Pf$ is $\epsilon$-connected} in $\gag$ if and only if every interval $I_{\A,\Ct}[\Pf[i], \Pf[j]]$ for $i,j \in N$ is $\epsilon$-connected  in $\gag$.
\end{definition}

While every interval in $\gag^c$ and $\gag^m$ is $\epsilon$-connected for every $\epsilon \geq 1$, this may not be the case for intervals in $\gagh$.
As an example, consider the doctrinal paradox profile $\Pf$ from Example~\ref{ex:doctrinal} and graph $\gagh$ in Figure~\ref{fig:doctrinalexample}a. Here, the interval $I_{\A,\Ct}[(1,1,1),(0,1,0)]$ is $2$-connected because every shortest path between $(1,1,1)$ and $(0,1,0)$ passes only through judgments that are not rational.
%\begin{observation} For a profile $\Pf \in \Dmc^n$,    $mx_{d_h}(\Pf) = mx_{d_h}(\CH(\Pf))$ and  $mx_{d_h}(\Pf) = mx_{d_h}(\CH(\Pf))$,  but it may be that   $mx_{d_m}(\Pf) < mx_{d_m}(\CH(\Pf))$.
%\end{observation}
%For the $\gag^c$ for every collection of distinct judgments $\Pf$ $mx_{d_c}(\Pf) =1$.  For $\gagh$, observe that the graph is a hypercube graph and every convex subset of the hypercube graph is itself a hypercube graph. For $\gag^m$ there is a counter-example.

\section{Iteration Algorithm}\label{sec:algorithm}

\new{Collective opinions in human societies are often formed not in one step, but rather in an intricate process that involves mutual information exchange, argumentation, persuasion, and opinion revision. Typically, social agents are motivated by two somewhat conflicting needs: on one hand, they want to form a unified stance with a significant enough part of the community; on the other hand, they do not want to concede too much of their own opinion. Here, we try to mimic this kind of behaviour -- obviously, in a very simplified way.}
To this end, we design an iteration algorithm, \alg,  based on an agenda graph $\gag$. As it is standard in judgment aggregation, we assume that the agents can only chose rational judgments at each iteration step.
For an agenda $\A$ and constraints $\Ct$, each agent's judgment is a node in the  graph $\gag$. In the first step of the iterative procedure each agent announces which node she has chosen. Two or more agents may choose the same node.  The agents do not have a constructed $\gag$ available. At each subsequent step, the agents (simultaneously)  compute their adjacent nodes and try  to ``move" to one of these adjacent nodes along some shortest path towards the other agents. A move is possible if and only if the adjacent judgment is rational and it brings the agent closer to the rest of the profile, \ie it decreases its aggregated path distance to the other judgments. More precisely, an agent $i \in N$ will move from $\Pf[i]$ to a  $\Js$  iff there exist a rational $\Js\in \CH(\Pf)$ s.t. $d(\Js, \Pf[i]) = 1$  and $\underset{j \in N, j\neq i}{\sum} d(\Pf[i], \Pf[j]) < \underset{j \in N, j\neq i}{\sum} d(\Js, \Pf[j])$, where $d$ is a path distance for a given $\gag$.  Given a choice between two moves the agent chooses the one which better reduces the distance to the rest of the profile. If more than one move reduces the distance to the same extent, the agent chooses  using some internal decision-making procedure which we do not model. We take it that in this case the agent chooses non-deterministically, with all move options being probable. The agents continue moving along the nodes of $\gag$  until no agent in the profile can move, or all of the agents ``sit" on the same node, namely until a {\em consensus} is reached.

\begin{algorithm}[t]\label{alg}
 \KwData{$\epsilon>0$,  $\Dmc$, $N$, own identifier $i \in N$, initial judgment $\Js_{i}^0$ }
 \KwResult{ $\Pf \in \Dmc^{|N|}$}
  $t:=0$;\ $\mbox{\moves}:= \emptyset$;\ $\Js_{i} := \Js_{i}^0$;\ $\Pf := \mbox{\em empty list}$\;

 \Repeat
{$\Pf$ is unanimous or  $\Pf' = \Pf$ }
%{$\mbox{unanimous}(\Pf^t)$}
{
$\Pf':=\Pf$\;
%construct $\Js_{i}$\;
$\Pf[i] := \Js_{i}$\;
\For{$j\in N, j\neq i$}{$send(\Js_{i}, j)$, $receive(\Js_{j},i),  \Pf[j]:=\Js_{j}$\;}
 $\textrm{\moves}:=  
    \argmin{\Js \in \CH(\Pf) \cap\Dmc}  D(i,\Js, \Pf)\quad \cap\quad $\linebreak
    $\mbox{}\quad\{ \Js \mid 0<d(\Js,\Js_{i})\leq~1\text{ and }D(i,\Js, \Pf) < D(i,\Pf[i], \Pf) \}$\;
 %\{ \Js  \mid   \dtp(\Js,\Pf_{-i}) < \dtp(\Js_{i},\Pf_{-i}), 0<~d(\Js,\Js_{i})~\leq~\epsilon\}$\;
 \If{ $\textrm{\moves}\neq \emptyset$ }
 { % $S_m =  \argmin{\Js \in \mbox{\small \em Moves} }\;\;  d(\Js, \Js_{i})$ \;
% \tcc{\small nondeterministic choice:}
 select $\Js \in \textrm{\moves}$,\quad $\Js_{i}:= \Js$\;}
%{$\Js:=\Js_{i}$\;}
$t:=t+1$\;
}
return $\Pf$\;

\smallskip
\caption{\marija{Iteration} algorithm}\label{alg:rra}
\end{algorithm}

In \alg, $send(\Js_{i}, j)$ informs agent $j \in N$ that agent $i \in N$ has chosen to move to node $\Js_i \in V$, while $receive(\Js_{j},i)$ denotes that the agent $i \in N$ has been informed that agent  $j \in N$  has chosen to move to node $\Js_j \in V$. To ease readability we use \linebreak  $D(i,\Js, \Pf) = \underset{j \in N, j\neq i}{\sum} d(\Js, \Pf[j])$. We call $D(i,\Js, \Pf)$ the distance of $\Js$ to the profile $\Pf_{-i}$.
In \alg, at each iteration $t$, \moves is the set of judgments that are strictly closer to $\Pf$ than the current judgment $\Js_i$. \new{Note that the algorithm is fully decentralised, in the sense that there is no need for any central authority to take over the iteration at any point of the process. }

The {starting profile} $\Pf^0$ collects the initial individual opinions of the agents. That is, it is the profile that would be aggregated under classical one-step social choice. We say that the algorithm \emph{reaches consensus} $\Js$ for $\Pf^0$, if it terminates starting from $\Pf^0$ by returning the unanimous profile $\{\Pf \}=  \{\Js\}$.
%Clearly, if the algorithm terminates with a consensus then the outcome is resolute, \ie no ties occur.
%When does the \alg~terminate?
We first observe a necessary condition for reaching consensus.
%termination is that $\Pf^0$ is $\epsilon$-connected.
\begin{proposition}\label{prop:necessary}
If \alg~reaches consensus, then $\Pf^0$ is $\epsilon$-connected for $\epsilon =1$.
\end{proposition}
\begin{proof}
Assume that the algorithm terminates  with a $\Js^{\ast}$-unanimous $\Pf$ at some $t$. In each  $t'<t$, every agent  either keeps  her own judgment $\Pf[i]$, or moves to a new $\Js\in\Dmc$ with  $d(\Pf[i], \Js) = 1$. Since a $\Js^{\ast}$ is reached by every agent, there  must exist a  $1$-connected path between any two  judgments in $\Pf^0$.   Thus $\Pf^0$ must be $1$-connected. \end{proof}

Note also that if $\Pf^0$ is $\epsilon$-connected, so is $\Pf$ at  any step $t > 0$.
The interesting question is: what are the sufficient conditions for reaching consensus by \alg?
We address the question in Section~\ref{sec:termination}.

%The choice of  $\epsilon$ clearly influences the behaviour of the algorithm. It is natural to assume that the amount of concession an agent is willing to make in order to reach a consensus should be as  small as possible. Since we deal with distance concession $\epsilon$, let us assume that $\epsilon = d_{min}$, where $d_{min}$ is the minimal non-null value that the function $d$ can take. In the rest of the paper we assume  $\epsilon = d_{min}$ and leave the case of $\epsilon > d_{min}$ for future work.

\section{Reaching Consensus}\label{sec:termination}

In this section, we examine the sufficient conditions for reaching consensus by \alg. We begin by looking at the iteration over the complete agenda graph $\gag^c$, and then we move on to the more interesting cases of $\gagh$ and $\gag^m$.

\subsection{Iteration with $\gag^c$}

%Let us first consider the iterated aggregation based on the sum of Drastic distances, with $\epsilon = d_{min} = 1$.
%Note that the necessary condition from Proposition~\ref{prop:necessary} for reaching consensus holds, as every $\Pf^0$ is 1-connected for $d_c$.
%
\begin{theorem}\label{th:c1}
If $\Pf$ contains a \emph{unique plurality judgment $\Js$}, then \alg always reaches consensus in one step.
\end{theorem}
\begin{proof} On $\gag^c$, the path distance $d_c$ between any two judgments that are different is $1$.
Let $\Js$ be the unique plurality judgment in $\Pf \in \Dmc^n$, selected by $k$ agents.
For every $\Js_i=\Js\neq \Js_j$, we have $ D(i,\Js, \Pf)
%D(J,\Pf_{-i})
=n-k
\le
% D(J_j,\Pf_{-i})
D(i,\Pf[j], \Pf)$, so the agents selecting $\Js$ can not change their judgments.
Moreover, switching from $\Js_j$ to $\Js$ decreases the distance to $\Pf_{-j}$ most, so all the other agents will switch to $\Js$ in the first iteration. \end{proof}

What about starting profiles with several plurality judgments? They converge towards consensus under reasonable conditions.
\begin{theorem}\label{th:c2}
If $N$ consists of an odd number of agents, then \alg probabilistically reaches consensus, \ie it reaches consensus with probability 1.
\end{theorem}
\begin{proof}
If there is a single plurality judgment in $\Pf$, then the algorithm converges in one step. If there are two or more plurality judgments $\Js^1,\dots,\Js^k \in \Dmc$, then those agents swap non-deterministically between $\Js^1,\dots,\Js^k$, and the other agents move to one of $\Js^1,\dots,\Js^k$. In the next round, the same argument applies. Eventually, $\Pf$ converges either to the unanimous profile $\Pf'$ such that  $\{ \Pf'\} = \{ \Js \}$ for some $\Js\in\{\Js^1,\dots,\Js^k\}$, or to a profile $\Pf''$ such that $\{ \Pf''\} = \{\Js^1,\dots,\Js^m\}$ for an odd $m$, each favoured by the same amount of agents. From then on, the agents keep swapping judgments until one judgment gets plurality in the profile, and wins in the next round.

\new{Formally, let $\moves_{i,\marija{t}}$ be the set of moves available to agent $i$ at \marija{ the step  $t$ of \alg}. We assume that there is some $\delta>0$ such that, for each \marija{step $t$}, agent $i$ selects judgment  $\Js\in\moves_{i,\marija{t}}$ with probability $p_i(\Js) \ge \delta$.
Then, there exists $\delta'>0$ such that the probability of all the agents ``hitting'' a profile with no plurality in the next round is at most $1-\delta'$. Hence, the probability that the profile stays with no plurality in $m$ steps is at most $(1-\delta')^m$, which converges to 0 as $m$ increases.} \end{proof}

The \alg has good convergence  properties on $\gag^c$ but the consensus it reaches is limited to the judgments that are already in the starting profile. On the Hamming and Model agenda graphs \alg surpasses this limitation. However, its convergence becomes a subtler issue.

\subsection{Iteration with $\gagh$ and $\gag^m$}\label{sec:termination-geodesic}

In this section, we will use $\gag$ to refer to one of $\gag^m,\gagh$ in order to avoid stating and proving the same properties for $\gag^m$ and $\gagh$ separately when the same proof can be applied.

We start with a negative result.
Let us call {\em equidistant} those profiles $\Pf$ such that, for any $i,j,r \in N$ with $i \neq j$, $i \neq r$, $j \neq r$, it holds that $d(\Pf[i], \Pf[j]) = d(\Pf[i], \Pf[r])$.

% end-putaway

\begin{proposition}
Consider $N=\{ 1, 2,3\}$ agents and a $1$-connected $\Pf^0$. If \alg reaches a $\Pf$ that is an equidistant $k$-cycle, then $\alg$ will not terminate with a consensus.
\end{proposition}
\begin{proof}
If $\Pf$ is an equidistant $k$-cycle then no agent can reduce the distance to one agent by 1 without increasing the distance to the other agent by 1. Thus no agent has a possible move.\end{proof}

\begin{figure}[h!]\centering
 \includegraphics[width=0.5\textwidth]{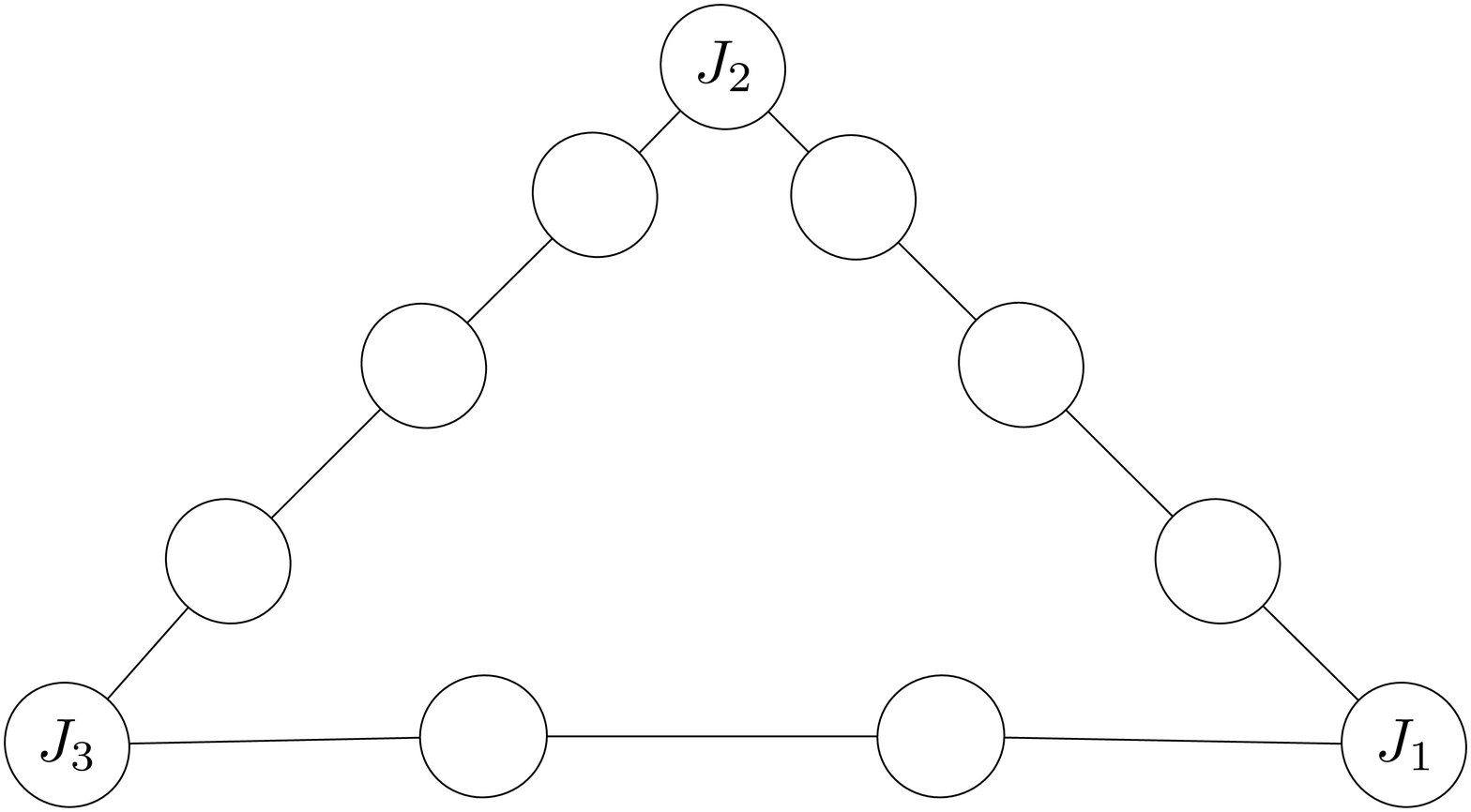}%{Triangle.eps}
 \caption{An agenda graph of a non-equidistant profile $\Pf = ( \Js_1, \Js_2,\Js_3 )$ }\label{fig:triangle}
\end{figure}

Note that the same applies to some non-equidistant $k$-cycles. For example, the profile in Figure~\ref{fig:triangle} is not equidistant, but it is easy to check that each agent has an empty set of moves.

3-agent profiles that form a simple cycle are problematic because an agent may not be able to get closer to one of the other agents without distancing itself from the third. For  profiles of more than three agents that form a simple cycle, the judgments cannot be equidistant, and this is no longer a problem.

\begin{lemma}\label{lem:cycleP}
If $\Pf$ is a (1-connected) $k$-cycle for $n > 3$ agents at step $t$ of \alg, then the set \moves at $t$ is nonempty for some  $i \in N$.
\end{lemma}
\begin{proof}
Take any judgment $\Pf[i]$ which is peripheral in $\CH(\Pf)$, and consider its antipodal judgment $\Pf[j]$. Let $\mathbf{p_1},\mathbf{p_2}$ be the two paths from $\Pf[i]$ to $\Pf[j]$ in $\CH(\Pf)$,
% NOTE: only one of them has to be a shortest path!
and let $\Js_m \in \CH(\Pf)$ be the node adjacent to $\Pf[i]$ on path $\mathbf{p_m}$, $m=1,2$.
We have that  $d(\Js_1,\Pf[r]) = d(\Pf[i], \Pf[r])-1$ for every $\Pf[r]$ on $\mathbf{p_1}$, while $d(\Js_2,\Pf[s]) = d(\Pf[i], \Pf[s]) +1$ for every  $\Pf[s]$ on $\mathbf{p_2}$. If there are more profile judgments on $\mathbf{p_1}$ than on $\mathbf{p_2}$, then $\Js_1 \in \textrm{\moves}$, otherwise   $\Js_2 \in \textrm{\moves}$. If there are exactly as many judgments on $\mathbf{p_1}$ as there are on $\mathbf{p_2}$, then both  $\Js_1 \in \textrm{\moves}$ and  $\Js_2 \in \textrm{\moves}$ because in that case  $d(\Js_1,\Pf[j]) = d(\Js_2, \Pf[j]) = d(\Pf[i],\Pf[j]) - 1$  and consequently $D(i, \Js, \Pf) = D(i, \Js', \Pf) < D(i, \Pf[i], \Pf)$. \end{proof}

Let us consider the case of a $1$-connected $\Pf^0$ for $|N| > 3$. Let $\Pf$ be the profile produced by \alg at step $t\geq 0$, and let $\Pf'$ be the profile produced by \alg at step $t+1$. We begin by showing that for graphs in which no judgment has a degree higher than 2, it is never the case that $\Pf=\Pf'$, \ie there exist at least one agent for which $\textrm{\moves} = \emptyset$ for $\Pf$.

\begin{lemma}\label{lem:mov}
Let $\Pf$ be a profile for $n > 3$ agents, produced by \alg at step $t\geq 0$, and let $\Pf'$ be the profile produced by \alg at step $t+1$. Assume that $\Pf$ is $1$-connected on $\gag$.   If  the $\CH(\Pf)$-induced subgraph of $\gag$ is such that no vertex in it has a degree higher than two, and $mx_d(\Pf) >1$, then  $\Pf \neq \Pf'$.
\end{lemma}

\begin{proof}~
We show that at least the agent $i \in N$  with a peripheral judgment for $\Pf$  has a possible move in $\Pf$.

\begin{description}
\item[Case a.] There exists a peripheral judgment  in $\Pf$ with degree 1, belonging to   $i\in N$.\\
Let $\Pf[j]$ be an antipodal of $\Pf[j]$.
Since $d(\Pf[i], \Pf[j]) >1$ and $\Pf$ is $1$-connected, there must exist exactly one judgment $\Js \in \Dmc$, such that $d(\Pf[i], \Js)=1$ and that is between judgments $\Pf[i]$ and $\Pf[j]$. For every $r \in N$, $r\neq i$ it holds that $d(\Js, \Pf[r]) = d(\Pf[i], \Pf[r])-1$. Thus $\Js$ is a move for $\Pf[i]$.

\item[Case b.] All peripheral judgments in $\Pf$ are with degree 2. Consider the antipodal judgments $\Pf[i]$ and $\Pf[j]$. There are exactly two shortest paths connecting them: $\mathbf{p_1}$ and $\mathbf{p_2} $. All other profile judgments $\Pf[r]$ are: either on  $\mathbf{p_1}$,  or on $\mathbf{p_2}$, or  have a shortest path to $\Pf[j]$ that intersects either $\mathbf{p_1}$  or $\mathbf{p_2}$, possibly both. We can apply the same reasoning as in the proof of Lemma~\ref{lem:cycleP}.

Consider $\Js \in \CH(\Pf)$ adjacent to $\Pf[i]$ on $\mathbf{p_1}$ and  $\Js' \in \CH(\Pf)$ adjacent to $\Pf[i]$ on $\mathbf{p_2}$. We have that  $d(\Js,\Pf[r]) = d(\Pf[i], \Pf[r]) -1$ for every     $\Pf[r]$ on $\mathbf{p_1}$ or whose shortest path to $\Pf[j]$ intersects $\mathbf{p_1}$, while $d(\Js,\Pf[s]) = d(\Pf[i], \Pf[s]) +1$ for every  $\Pf[s]$ on $\mathbf{p_2}$ or whose shortest path to $\Pf[j]$ intersects $\mathbf{p_2}$, but does not intersect $\mathbf{p_1}$. If there are more agents $r$ than agents $s$, then $\Js \in \textrm{\moves}$, otherwise   $\Js' \in \textrm{\moves}$. If there are exactly as many agents $r$ as agents $s$, then both  $\Js \in \textrm{\moves}$ and  $\Js' \in \textrm{\moves}$ because in that case  $d(\Js,\Pf[j]) = d(\Js', \Pf[j]) = d(\Pf[i],\Pf[j]) - 1$  and consequently $D(i, \Js, \Pf) = D(i, \Js', \Pf) < D(i, \Pf[i], \Pf) $.\end{description} \end{proof}

Observe that if the $\CH(\Pf^0)$-induced subgraph on $\gag$ is such that every vertex in it is of degree at most two, then for every subsequently constructed $\Pf$ in \alg, it will hold that the $\CH(\Pf)$-induced subgraph on $\gag$ is such that every vertex in it is of degree at most two. This is due to the fact that, at each step of \alg, the agents can only chose judgments from the   $\CH(\Pf^0)$.

From Observation~\ref{col:cycle} we know that a   $\{\Pf\}$-induced subgraph of $\gagh$ and $\gag^m$ does not have 3-cycles. If the $\CH(\Pf^0)$-induced subgraph  of   $\gagh$, respectively $\gag^m$ contains no $k$-cycles for $k > 3$, then this induced subgraph contains no cycles and it is by definition a tree. From the Case a. of the proof of Lemma~\ref{lem:mov}, we immediately obtain the following corollary.

\begin{corollary}\label{lem:tree} Let $\Pf$ be a profile produced in \alg at step $t\geq 0$ and let $\Pf'$ be profile produced in \alg at step $t+1$. Assume that $\Pf$ is $1$-connected on $\gag$.   If  the $\CH(\Pf)$-induced subgraph of $\gag$ is a tree, and $mx_d(\Pf) >1$, then  $\Pf \neq \Pf'$.
\end{corollary}

\begin{proof} The proof follows from the Case a. of the proof of Lemma~\ref{lem:mov} and the observation that: all the subgraphs of a tree are trees,  and the peripheral vertices of a tree have a degree 1. \end{proof}

We now need to show that not only the profile changes in each iteration, it also changes towards a consensus.

From Proposition~\ref{prop:mxd} we have that $mx_d(\CH(\Pf)) $ does not increase with each step of the \alg. It is possible that $mx_d(\CH(\Pf)) = mx_d(\CH(\Pf'))$ for $\Pf'$ being constructed  immediately after $\Pf$ in \alg. From the proof of  Lemma~\ref{lem:mov} we have the following corollary.

\begin{corollary}\label{lem:mcyc} Let $\Pf \in \Dmc^n$ be a profile produced in \alg at step $t\geq 0$ and let $\Pf' \in \Dmc^n$  be profile produced in \alg at step $t+1$.  If  $\{\Pf\} = \{\Pf'\}$, then the $\{\Pf\}$-induced graph of $\gag$   has  at least one    $k$-cycle, where $2m+ 2 \geq k \geq 2m$.
\end{corollary}

Clearly if the agents whose judgments are antipodal in $\Pf$  can  choose to move towards each other via two different shortest paths between their judgments causing $\{\Pf\} = \{ \Pf'\}$. These agents however, also have the possibility to chose to move towards each  other on the same shortest path between their judgments.  As soon as two agents use the same shortest path, the $k$-cycle will be broken in the next step of the algorithm and $\{\Pf\} \neq \{ \Pf'\}$.

Let us consider the case when $mx_d(\Pf) =1$.

\begin{lemma}\label{lem:odd} Let $\Pf$ be a 1-connected profile for $n > 3$ agents at step $t$ with $mx_d(\Pf) =1$ and let $\Pf'$ be a profile obtain from it  by \alg at step $t+1$. If $n$ is odd then $\{ \Pf\} \neq \{ \Pf'\}$.
\end{lemma}
\begin{proof} In this case   the \alg behaves as on the $\gag^c$ graph, see  Theorems~\ref{th:c1}~and~\ref{th:c2}, except the $\Pf$-induced subgraf of $\gag$ will have no 3-cycles (or any size cycles since  $mx_d(\Pf) =1$). Namely, if there is one plurality judgment $\Js$ in $\Pf$, all the agents can reach it, because $mx_d(\Pf) =1$ and $\Pf$ is 1-connected. Consequently $\{\Pf'\} = 1$. If more than one plurality judgment exists, the agents whose judgment is this plurality judgment will not have a move, while and all the other agents will move to their choice of a plurality judgment. If $n$ is odd $\Pf'$ will have exactly one plurality judgment and the profile $\Pf''$ constructed by \alg in step $t+1$ is a consensus. If however $n$ is even,  as with $\gag^c$, $\Pf$ can be such that   half of the agents have  a judgment $\Js$, while the other half have an adjacent judgments $\Js'$.  Namely   $\{\Pf\} = \{ \Js, \Js'\}$ and $d(\Js,\Js') =1$. If such $\Pf$ is reached the \alg forces the agents to infinitely ``swap" between $\Js$ and $\Js'$. \end{proof}

  \begin{lemma}\label{lem:prob} Assume an odd number of agents $n >3$. that in the $\Pf$-induced subgraph on $\gag$ each vertex has a degree at most 2.  Let   $\Pf \in \Dmc^n$ be  s.t. $\CH(\Pf)$ has  at least one   $k$-cycle for  $k >3$. Let $p_{i}(\Js) > 0$ be the probability that an agent $i$ will choose a possible move $\Js$ from the set $\textrm{\moves}$  at a step $t_1$ in the \alg. Then the algorithm will reach a point $t_2 > t_1$ where $\Pf' \in \Dmc^n$ is obtained s.t. $\CH(\Pf) \subset \CH(\Pf')$ with probability 1.
 \end{lemma}
 \begin{proof} If a profile $\Pf''$ is reached such that  all antipodal judgments have degree two, it is sufficient that only one antipodal pair  ``breaks" the cycle for a profile $\Pf'$ to be reached. To do so, two agents with antipodal judgments have to chose to move along the same path towards each other.
 Consider a  pair  of antipodal judgments in $\Pf$, $\Pf[i]$ and $\Pf[j]$.  Assume that at the non-deterministic step of the algorithm there exists a probability $1>p_{i}(\Js) >0$ that the agent $i$ selects $\Js\in \textrm{\moves}$ that is on a shortest path $\textbf{p}$ between $\Pf[i]$ and $\Pf[j]$ and probability $p_{i}(\Js') = 1- p_{i}(\Js)$ that she selects  $\Js'\in \textrm{\moves}$ that is on a different path  $\textbf{q}$  between $\Pf[i]$ and $\Pf[j]$. Similarly, let those  probabilities be  $1>p_{j}(\Js'') >0$ that agent $j$ selects to move to $\Js''$ on  path $\textbf{p}$ and $p_{j}(\Js''') = 1- p_{j}(\Js'')$ for the probability that  $j$ moves to $\Js'''$ on some other path $\textbf{q'}$ ($\textbf{q}$  and $\textbf{q'}$   may not be the same). Since the agents decide on their moves independently, the probability that agent $i$ will chose the same   path as $j$ is $p_{i}(\Js) \cdot p_{j}(\Js'')>0$. Since the two  peripheral judgments $\Js[i]$ and $\Js[j]$ are no longer part of the new profile $\Pf'$, $\{\Pf' \} \subset \CH(\Pf)$ and from Corollary~\ref{lem:mcyc} we get that  $\CH(\Pf) \subset \CH(\Pf')$ is reached after a finite time with probability 1. \end{proof}

Let us call {\bf Class A for $\gag$}  the set of all $\CH(\Pf)$-induced subgraphs of $\gag$ that are tree graphs. Let us call {\bf Class B for $\gag$} the set of all $\CH(\Pf)$-induced subgraphs of $\gag$  whose vertices have a degree of at most 2. 
For instance, the doctrinal paradox profile from Example~\ref{ex:doctrinal} is in Class B for $\gag^m$, see Graph b in Figure~\ref{fig:doctrinalexample}. On the other hand, it is neither in Class B nor in Class A for $\gagh$, see Example~\ref{ex:hull} and Figure~\ref{fig:hull}.

We can now state the following theorem whose proof follows from  Lemma~\ref{lem:mov}, Corollary ~\ref{lem:mcyc}, Corollary~\ref{lem:tree}, Lemma~\ref{lem:odd} and Lemma~\ref{lem:prob}.

\begin{theorem}
Let $\Pf^0 \in \Dmc^n$  be a  $1$-connected profile belonging to Class A or to Class B for $\gagh$ or $\gag^m$. If $n > 3$ is odd, and each element of $\textrm{\moves}$ has a non-null probability of being selected in the non-deterministic choice step,  then the \alg reaches consensus with probability $1$ on $\gagh$, respectively $\gag^m$.
\end{theorem}

%===========================================================
\section{Properties of Consensus}\label{sec:properties}

In this section, we compare the output and performance of our iteration procedure to those of standard distance-based judgment aggregation rules.
We first discuss the ``quality'' of the consensual decision. Then, we look at the computational complexity of the procedure.

\subsection{Consensus Quality}

Distance-based judgment aggregation~\cite{MillerOsherson08,thesis,TARK11,JamrogaS13} combines an algebraic aggregation function $\star$ with a distance function $d$ (not necessarily a path distance in some agenda graph) in order to select the collective opinion that is closest to the given profile.
%Typically, $\star \in \{\sum, \max\}$
Given $\Pf\in \Dmc^n$, the distance-based aggregation function $F^{d,\star}: \Dmc^n \rightarrow 2^{\Dmc}\setminus \emptyset$  is defined as
$$F^{d,\star} (\Pf) = \argmin{\Js \in \Dmc}\;\; \star(d(\Pf[1], \Js), \ldots, d(\Pf[n], \Js)).$$
Natural questions to ask are:
\begin{itemize}
\item How does \alg perform in comparison to $F^{d,\Sigma}$ when $d$ is a path distance in an agenda graph?
\item How do the collective judgments $F^{d,\Sigma} (\Pf^0)$ compare to the consensus judgment reached by \alg for a given starting profile $\Pf^0$?
\end{itemize}
The questions cannot be fully explored within the scope of this paper. However, we establish some initial properties below.

\putaway{
  \begin{proposition}It is not always the case that, for a $1$-connected   $\Pf^0 \in \Dmc^n$,   if $F^{d,\Sigma} (\Pf^0) = \Js$, then $\alg$ will reach $\Pf$ with $\{ \Pf\} = \{\Js\}$.
  \end{proposition}
  \begin{proof} Consider the $\A$ and $\Ct$ with   $\gag$ in Figure~\ref{fig:cycle} and the profile $\Pf^0 \in \Dmc^n$, $\Pf^0 = \langle \Js_1, \Js_2, \Js_3, \Js_4 \rangle$. The \alg will not terminate with a consensus for this  $\Pf^0$, but $F^{d,\Sigma} (\Pf^0) = \Js$.
  \end{proof}
}% end-putaway

A property generally deemed desirable in judgment aggregation is that of \emph{propositional unanimity}~\cite{TARK11,thesis,GrandiEndriss13}. Propositional unanimity requires that, if every agent in profile $\Pf$ has the same value for some issue $\ai \in \A$, then the same value for $\ai \in \A$ shows up either in at least one of the judgments in $F^{d, \star}$ (weak unanimity) or in all of the judgments in $F^{d, \star}$ (strong unanimity). It is interesting to note that the most popular distance based judgment aggregation rule $F^{d_h, \Sigma}$ does not satisfy even the weak version of the property~\cite{ADT09} and the same applies to $F^{d_m, \Sigma}$ and $F^{d_c, \Sigma}$~\cite{LangPSTV15}. In this respect, iterative consensus building behaves better.

\new{
\begin{proposition}\label{prop:unanimity-H}
If \alg terminates with a consensus on $\gagh$, then the consensus satisfies strong unanimity with respect to the initial profile $\Pf^0$.
\end{proposition}
\begin{proof}
%Note that, for $\gagh$ and $\gag^m$, if judgment $\Js'$ is between judgments $\Js$ and $\Js''$,  then vertex $\Js'$ is on the shortest path between vertices $\Js$ and $\Js''$ in the graph. For $\gagh$, the implication actually goes in both directions.
Note that, for $\gagh$, judgment $\Js'$ is between judgments $\Js$ and $\Js''$ iff vertex $\Js'$ is on the shortest path between vertices $\Js$ and $\Js''$ in the graph.
Consequently, if all the agents in $\Pf^0$ give the same truth-value on an issue, then $\CH(\Pf^0)$ cannot contain judgments that assign different truth-value to this issue.
\end{proof}

The same is not the case for $\gag^m$.

\begin{proposition}\label{prop:unanimity-M}
There is an initial profile $\Pf^0$ such that \alg terminates with a consensus on $\gag^m$, and the consensus does not satisfy weak unanimity with respect to $\Pf^0$.
\end{proposition}
\begin{proof}
As a counter-example consider   Example~\ref{ex:doctrinal} and Graph $b$ in Figure~\ref{fig:doctrinalexample}. The vertex $(1,1,1)$ is between vertices $(0,1,0)$ and $(1,0,0)$, but the judgment $(1,1,1)$ is not between judgments $(0,1,0)$ and $(1,0,0)$. Thus the agents can move from $(0,1,0)$ and $(1,0,0)$ to  $(1,1,1)$  thus violating propositional unanimity on the last issue.
\end{proof}

A big advantage of one-shot distance-based aggregation $F^{d,\Sigma}$ is that it produces output (a winner or winners) on any profile $\Pf^0$, while our \alg is more restricted in this respect. As we have seen, a necessary condition for successful termination of \alg is that $\Pf^0$ is 1-connected. Sufficient conditions are even more restrictive.
Still, Proposition~\ref{prop:unanimity-H} demonstrates that, when \alg reaches a consensus, it is structurally ``better behaved'' then a distance-based judgment aggregation rule for the most popular approach based on the sum of Hamming distances.
In the next subsection we show that \alg is also ``better behaved'' in the sense of computational complexity.
}% end-new

\subsection{Complexity of Reaching Consensus}

\new{An important drawback of distance-based judgment aggregation is the computational complexity of producing the output, \ie the winning judgment or judgments.
The winner determination problem for $F^{d_h,\Sigma}$ is known to be $\Theta^p_2$-complete~\cite{EndrissGP12}, and the result extends to most other distances $d$ and aggregation functions $\star$~\cite{JamrogaS13}. The computational complexity of determining the collective judgment sets by $F^{d_m,\Sigma}$ is actually not known.
How does it work for the iteration procedure formalized with \alg?}

We have shown that the algorithm reaches consensus for an odd number $n>3$ of agents on $1$-connected, not equidistant profiles. How costly is it to reach the consensus? On $\gag^c$, it is evident that \alg performs well, but the resulting consensus is not very exciting.  For the $\gag^m$ graph, the consensus-friendly attitude may not earn much in terms of computational complexity, when compared to  $F^{d_m,\Sigma}$.
For each $\Pf[i]$, we need to find every $\Js \in \Dmc$ s.t. there is no rational judgment between $\Pf[i]$ and $\Js$. It is not difficult to show, by a reduction to coSAT, that checking whether there is no rational judgment between two given rational judgments is in general coNP-complete.
This has to be repeated for multiple candidate judgments to compute the set \moves, and on top of that with every iteration of the algorithm. As a consequence, we get the following.

\begin{theorem}
For $\gag^m$, determining \moves of a single iteration of \alg   coNP-hard.
\end{theorem}

%We conjecture that the problem is actually $\Theta_2^P$-complete.
Note that the hardness result is not really due to the iteration procedure, but rather due to the inherent complexity of computing $d_m$, which requires to determine nonexistence of particular rational judgments, \ie to solve the Boolean co-satisfiability problem.

In contrast, the Hamming distance $d_h$ can be always computed efficiently. Consequently, when \alg reaches a consensus on $\gagh$, it is also ``better behaved" computationally than the distance-based judgment aggregation rule $F^{d_h,\Sigma}$. We demonstrate it formally.

\begin{proposition}
For $\gagh$, a single iteration of \alg runs in deterministic polynomial time.
\end{proposition}
\begin{proof}
Follows from the fact that the set \moves can be constructed by checking at most ${|\A|}$ candidate judgments.
\end{proof}

By Corollary~\ref{lem:mcyc}, if the $\CH(\Pf^0)$ induced subgraph of $\gagh$ has no cycles, then the diameter of $\CH(\Pf)$ is strictly shrinking with each non-terminating step $t$. In consequence, if \alg reaches consensus for such $\Pf^0$, then it does so in polynomially many steps.
However, in case of cycles in the $\CH(\Pf^0)$ induced graph  in $\gagh$, the algorithm may run into such a cycle and take some time until the agents ``stray'' from the loop. When it happens, any judgment occurring on the loop can be the consensus.
%Assume the probability of all the agents staying on the cycle in one round to be $p$. Then, the expected number of iterations until they stray from the cycle is $\sum_{i=1}^\infty i\cdot p^{i-1}\cdot (1-p) = ????$
Using this observation, we propose the following modification of \alg.

\para{\algtwo:} Same as \alg, \new{only it stops the iteration when  $\{\Pf_t\}= \{\Pf_{t'}\}$ for some $t>t'$}, and non-deterministically chooses one $\Js\in\{\Pf\}$ as the consensus, producing in the next step $\Pf_{t+1}$ with $\{\Pf_{t+1}\} = \{ \Js\}$.

\new{Unlike \alg, \algtwo avoids looping and waiting until two or more agents ``move'' in the same direction.
It also avoids infinite loops in the case of profiles with evenly many agents.
%because as soon as  an even number of the agents  split into two same size groups, one choosing  $\Js$ and the other an adjacent $\Js'$, the algorithm will terminate and one of $\Js$ and $\Js'$ will be  chosen as a consensus at random.
On the other hand, \algtwo is no longer decentralised, which is a clear disadvantage.
We suggest that it can be treated as a technical variant of \alg that potentially reduces its running time by employing a trusted third party which simulates probabilistic convergence of the profile in \alg by one-step non-deterministic choice in \algtwo.}
The following formal results, which are straightforward consequences of our analysis above, justify the suggestion.

\begin{theorem}
Consider $\gagh$ and $N$ such that $|N|$ is odd and larger than 3.
\new{If \alg can reach consensus with $\Js$ then also \algtwo can reach consensus with $\Js$.}
\end{theorem}

\begin{theorem}
Consider $\gagh$ and $N$ such that $|N|$ is odd and larger than 3. Moreover, let $\Pf^0 \in \Dmc^n$ be 1-connected and not equidistant.
\algtwo reaches consensus from $\Pf^0$ on $\gagh$ in deterministic polynomial time.
\end{theorem}

% WE ACTUALLY DON'T WANT TO DISCUSS THE FOLLOWING:
%We here discuss how computationally difficult it is to determine whether the algorithm will terminate for a given $\Pf^0$.
%For the Hamming distance, checking whether $\Pf^0 \in \Dmc^n$ is $1$-connected can be done in polynomial time with respect to  $n \times |\A|$. Namely,  every judgment in $\Dm$ has exactly $|\A |$ judgments at Hamming distance 1. For each agent $i$ we need to check which of the judgments at Hamming distance 1 from $\Js^0_i$ are rational. Note that for $d_h$,  when  $\Pf^0 \in \Dmc^n$ is $1$-connected, the consensus can be found in polynomial time with respect to  $n \times |\A|$.
%For the distance $d_m$, by definition, every profile is  $1$-connected. However, the problem of finding the consensus, \ie running the algorithm, is no longer polynomial, because determining the possible moves for an agent can no longer be done in polynomial time. Namely, Since $n$ and $|\A|$ are fixed, finding a path between two rational judgments is also in  {\sc NP}: guessing a path and checking that each judgment in it is rational. Thus  determining the possible moves for an agent  would require no less than $n-1$ calls to an {\sc NP}-oracle, one for each  agent in the profile.

Lastly, let us observe that checking whether $\Pf^0$ is equidistant can be done in linear time of the number of agents.
For a graph $G$, determining if  it has a simple cycle of size $k$, for $k$ fixed,  is a  polynomial time problem over the size of $G$, see \cite{Alon:1997}, however we do not generate the full $\gag$ when (or before) we run the iteration algorithm.

%===========================================================

\section{Related Work}\label{sec:rwd}

List \cite{List2011} considered judgment transformation functions $\tau: \Dmc^n \rightarrow (\Dm)^n$ as means to building iteration procedures for judgment aggregation problems. He showed that for a set of desirable properties no transformation functions exists. Such impossibility results exist for judgment aggregation functions, however, by relaxing some of the properties, specific judgment aggregation operators have been constructed: quota-based rules \cite{DietrichList07}, distance-based rules \cite{Pigozzi2006,MillerOsherson08,EndrissGP12,ConalPiggins12}, generalisations of Condorcet-consistent voting rules \cite{PuppeNehring2011,NehringPivato2011,TARK11},  and  rules  based on the maximisation of some scoring function \cite{TARK11,Dietrich:2013,Zwicker11}.  To the best of our knowledge,  specific iteration procedures for judgment aggregation problems have not been proposed in the literature.

List \cite{List2011} argues that the desirable conditions for judgment  transformation functions  should satisfy the following properties:  universal domain, rational co-domain, consensus preservation,  minimal relevance, and independence. Universal domain is satisfied when the  transformation function accepts as admissible input any possible profile of rational judgments. Rational co-domain is satisfied when the function always outputs a profile of rational judgments. Consensus preservation is satisfied when $\tau$ always maps unanimous profiles into themselves. Minimal relevance is a weak property. It is satisfied when for each  $\Pf[i]$ there exists a profile $\Pf'$ to which $\Pf$ can be transformed such that  $\Pf[i] = \Pf'[i]$. In other words, the transformation function should be such that it does not allow one agent to never change her judgment regardless of what the other profile judgments are. Lastly  independence is satisfied when for each agenda issue,  $\Js'_i(\ai)$  depends only on $\Js(\ai)$, and not on $\Js(\phi)$ for some other $\phi \in \A$;  $\Js'_i = \Pf'[i]$, $\Js_i=\Pf[i]$, $\Pf' = \tau(\Pf)$.

Each step of \alg can be seen as a (distributed) function that transforms an input profile into an output profile, namely as a List judgment transformation function.  Given a profile $\Pf \in \Dmc$, let $T_d(\Pf) = \{ \Pf' \mid \Pf'  \in \Dmc^n, d(\Pf[i], \Pf'[i]) \leq 1, i\in[1,n]\}$. We can define the transformation function $\tau_d$ that maps a profile $\Pf \in \Dmc^n$ to a profile $\Pf' \in T_d(\Pf)$. Although  \alg does not terminate  for each profile,  $\tau_d$ does satisfy universal domain in the case of $\gag^c$ and $\gag^m$, because each step of the algorithm transforms the profile (possibly into itself). Universal domain is not satisfied on $\gagh$ because profiles on this graph do not always satisfy the necessary conditions for termination with a consensus. The rational co-domain and the consensus preservation properties are also trivially satisfied. It is not difficult to show that the minimal relevance property is also satisfied.  Independence is the desirable property that is violated, and in fact List \cite{List2011} argues that relaxing independence is the most plausible path towards avoiding the impossibility result.

In voting, deliberation and iterative consensus have been studied, although perhaps not axiomatically. As most similar with our work we distinguish \cite{Hassanzadeh2013} and \cite{Goel2012}.   Voting problems can be represented as judgment aggregation problems, see \eg \cite{DietrichList07,ADT2013}, therefore it is possible to compare these works with ours. First we show how voting problems are represented in judgment aggregation.

A voting problem is specified with a set of agents $N$ and a set of candidate options $O=\{x_1, x_2, \ldots, x_m\}$. Let $\mathcal{O}$ be the set of all total, transitive, and antisymmetric orders over the elements of $O$. A vote $\succ$ is an element of $\mathcal{O}$ and a voting profile is a collection of votes, one for each agent in $N$.   The preference agenda $\A_o$ is constructed by representing each pair of options $x_i$ and $x_j$, where $i < j$ with an issue $x_iPx_j$. The constraint $\Ct_{\textrm{tr}}$ is the transitivity constraint defined as \linebreak $\Ct_{\textrm{tr}}=\underset{x_iPx_j, x_jPx_k, x_iPx_k \in \A_o}{\bigwedge} \big((x_iPx_j) \wedge (x_jPx_k) \rightarrow (x_iPx_k)\big)$. For each vote $\ai \in \mathcal{O}$ we obtain a rational judgment $\Js_{\succ}$ such that $\Js_{\succ}(x_iPx_j) = 1$ iff $x_i \succ x_j$ and  $\Js_{\succ}(x_iPx_j) = 0$ iff $x_j \succ x_i$.

\marija{A Condorcet winner for a voting profile, when it exists,  is the option that wins the majority of  pairwise comparison for every other option in $O$, see \eg \cite{Nurmi10}.  The corresponding concept in the  judgment aggregation representation of a voting problem is called {\em majority consistency}. A judgment profile is majority-consistent if the judgment obtained by taking  the value for each issue assigned by a strict majority of agents in the profile is rational. The doctrinal paradox profile from  Example~\ref{ex:doctrinal} is not majority-consistent.  It was shown \cite{ADT2013,NehringPP14} that if a judgment profile on the preference agenda is majority-consistent, then the corresponding voting profile has a Condorcet winner. }

Hassanzadeh {\em et al.} \cite{Hassanzadeh2013} consider an iterative consensus algorithm for voting profiles. In their algorithm, each agent is allowed to (simultaneously with other agents) move from vote $\succ_i$ to vote $\succ$ if she can flip the order of two adjacent options without violating transitivity.  This corresponds to the agents moving to an adjacent judgment in the agenda graph $G^h_{\A_o, \Ct_{\textrm{tr}}}$.  Hassanzadeh {\em et al} consider the majority graph for a voting profile (for an odd number of agents): the vertices in this directed graph are the elements of $O$ and there is an edge from $x_i$ to $x_j$ if  there are more agents in the profile who prefer $x_i$ to $x_j$ than agents who prefer $x_j$ to $x_i$. The majority graph corresponds to a judgment $\Js\in \mathcal{J}_{\A_o, \Ct_{\textrm{tr}}}$ for which  $\Js(x_iPx_j) = v$, $v \in \{0,1\}$ if there is a strict majority of   agents $r \in N$  for which  $\Js_r(x_iPx_j) = v$, $\Js_r = \Pf[r]$. Hassanzadeh {\em et al}  show that their algorithm terminates with a consensus on the Condorcet winner when the majoritarian graph has no cycles. \marija{If the majority graph of a voting profile has no cycles, then the voting profile has a Condorcet winner.}

Goel and Lee  \cite{Goel2012} consider an iteration procedure in which the agents ``move" along adjacent vertices along (what corresponds to) the graph $\gag^m$. They do not commit to the nature of their vertices, so they are not exactly judgments or alternatives, just allowed options for iteration. In their algorithm not all agents  move individually, but three agents at a time first reach a consensus and then all three move to the consensus option in the graph. Goel and Lee  consider line graphs, graphs in which two vertices have degree 1 and all other vertices have degree 2, and show that the consensus produced by their algorithm is  the generalised median. Namely, if the options in their algorithms were judgments from $\Dmc$  the consensus their algorithm reaches for these graphs is an approximation of  $F^{d_m,\Sigma}$.

Both \cite{Hassanzadeh2013} and  \cite{Goel2012}  offer interesting directions for future study in context of our algorithm: to consider the profiles that have a Condorcet winner (see \eg \cite{ADT2013} for the concept of Condorcet winner in judgment aggregation) and to consider triadic iteration, allowing three agents to coordinate their moves with respect to each other and then see when a consensus emerges.

It is an open question of how our algorithm would perform on the special case of voting problems represented in judgment aggregation. The \marija{ $G^h_{\A_o, \Ct_{\textrm{tr}}}$ graph  on the preference agenda has a more regular topology in comparison to general judgment aggregation problems, it is a permutahedron.} \marija{ For example, for an agenda of three options, the  graph $G^h_{\A_o, \Ct_{\textrm{tr}}}$  is a cycle of length 6.} For every $\Js, \Js' \in \mathcal{J}_{\A_o, \Ct_{\textrm{tr}}}$, $d_h (\Js, \Js') = d_m (\Js, \Js')$, thus the necessary conditions for reaching consensus for \alg would be satisfied even on $\gagh$ because  every profile on the preference agenda and judgments rational for the transitivity constraint  is $1$-connected in  $G^h_{\A_o, \Ct_{\textrm{tr}}}$. The graph  $G^m_{\A_o, \Ct_{\textrm{tr}}}$ always  has $\frac{|O|\cdot (|O|-1)}{2}$ vertices and each of these vertices  has a degree $|O|-1$.  We leave for future work the study of  whether our algorithm terminates for voting profiles. In particular, we conjecture that the \alg for an odd number of agents will converge on the Condorcet winner in the case of voting profiles from the single crossing domain \cite{Bredereck2012}, also studied in judgment aggregation \cite{domains}. This is because  profiles in this domain would have a hull whose induced graph is a line on $G^h_{\A_o, \Ct_{\textrm{tr}}}$.

\marija{Lastly, we must mention \cite{Obraztsova:2013}.  Obraztsova {\em et al} \cite{Obraztsova:2013} consider a graph similar to our Hamming agenda graph. They work with preferences, not judgments, but most importantly,   the vertices of their graph are elements of (what would correspond to)   $\Dmc^n$, \ie the vertices are profiles of votes. There exists a connection between two profiles if one profile can be obtained from the other by making exactly one swap between adjacent options in one vote.  Obraztsova {\em et al}   \cite{Obraztsova:2013} study the properties of voting rules with respect to the ``geometry" of the profiles in their graph. }

%===========================================================
\section{Conclusions}\label{sec:sum}

In this paper, we propose a decentralised  algorithm for iterative consensus building in judgment aggregation problems. 
% Since judgment aggregation is an abstract framework, it allows for other social choice aggregation problems, such as voting, to be represented in it, see \eg  \cite{aaai} for a recent overview. 
We study the termination conditions for this algorithm, some of its structural properties, and its computational complexity.

In order to reach a consensus, our algorithm exploits the topology of a graph. All the available judgments that the agents can chose from are vertices in the graph. The algorithm models an agent's change of mind as a move between adjacent judgments in the graph.  We define three natural graphs that can be constructed for a set of rational judgments $\Dmc$: the complete graph $\gag^c$, the Hamming graph $\gagh$, and the model agenda graph $\gag^m$. We prove that our algorithm always terminates for an odd number of agents on the graph $\gag^c$, but it necessarily selects one of the judgments proposed in the first round of iterations. For the graphs  $\gagh$  and $\gag^m$ we show a class of profiles for which the algorithm terminates with a consensus and a class of profiles for which it does not terminate with a consensus.

%If the profile is equidistant, namely each agent has  chosen judgments with the same path distance to each of the judgments of the other agents, then the \alg does not terminate with a consensus.
If the agents initially chose judgments such that the convex hull of the profile of these judgments  induces a subgraph of $\gagh$,  or  $\gag^m$ in which each vertex has a degree of at most 2, then our algorithm probabilistically terminates with a consensus for an odd number of (more than 3) agents.

The list of  profiles we give here,  for which  \alg terminates with a consensus,  is clearly not exhaustive. For example, it is easy to show that, for an odd number of agents,  \alg terminates with a consensus if the $\CH(\Pf^0)$ induced subgraph  of $\gagh$,  or  $\gag^m$, is such that it contains only $k$-cycles, where $k= 2\cdot mx_d(\Pf^0) +1$. This is because for such profiles there exists at least one pair of antipodal judgments with degree no more than 2 who will have a nonempty set \moves. %This means that the starting profile $\Pf^0$ would be allowed to only have cycles of the same size. The doctrinal paradox profile has a convex hull whose induced subgraph of $\gagh$ satisfies this condition as it can be directly observed in Figure~\ref{fig:hull}.
An immediate direction for future work is to strengthen our results with other classes of consensus terminating agenda graph topologies, particularly those corresponding to profiles on the preference agenda (and transitivity constraints).

A step of our algorithm implements  a judgment profile transformation function of the type defined in \cite{List2011}. List \cite{List2011} gives an impossibility characterisation of such functions. Our function ``escapes" this impossibility result by not satisfying  the independence property  on all agenda graphs and the universal domain on $\gagh$.

While  $\gag^c$,   and $\gag^m$ satisfy the necessary conditions for termination of \alg for any $\A$ and $\Ct$, this is not the case with $\gagh$, which is why the transformation function fails to satisfy universal domain on $\gagh$. On  $\gagh$, sometimes all the adjacent judgments to a rational judgment $\Js$ are not rational and thus not allowed to move to. In our future work we aim to explore modifications of the algorithm allowing the agents to make ``longer" moves, \ie to ``jump over" a vertex that is not a rational judgment.

In Section~\ref{sec:properties} we gave two results with respect to the quality of the  consensus reached by \alg with respect to the widely used distance-based aggregation function $F^{d_h, \Sigma}$.  This function $F^{d_h, \Sigma}$ is also known as the median aggregation rule and it is widely used in many domains, \eg  generalises the Kemeny voting rule, see \cite{ADT2013},  and for measuring dissimilarity between concepts in ontologies \cite{DistelAB14}. We merely scratched the surface of this consensus quality analysis and this line of research merits further attention.

Lastly, a more long-term goal for our future work is to explore versions of iteration on an agenda graph where the agents do not try to move to reduce the path distance to all of the other agents, but only to their neighbours in a given social network, or as in \cite{Goel2012}, to two randomly selected two agents.

%===========================================================

\smallskip
\noindent\textbf{Acknowledgements.} We are grateful to Edith Elkind and the anonymous reviewers for their help in improving this work. Wojciech Jamroga acknowledges the support of the 7th Framework Programme of the EU under the Marie Curie IEF project ReVINK (PIEF-GA-2012-626398). We also acknowledge the support of ICT COST Action IC1205.
%\clearpage
\bibliography{ecai}

\end{document}